\relax
\documentclass[letterpaper]{article}
\usepackage{url}

\usepackage{fp,xcolor,colortbl}
\usepackage{algorithm}
\usepackage{algorithmic}
\usepackage{amssymb}
\usepackage{amsthm}
\usepackage{booktabs}
\usepackage{lmodern}
\usepackage{mathtools}
\usepackage{nth}
\usepackage{graphicx}
\usepackage{subfigure}

\usepackage{iclr2017_conference}
\usepackage{times}
\usepackage{helvet}
\usepackage{courier}
\makeindex

\def\X{{\bf X}}

\def\W{{\bf W}}

\def\0{{\bf 0}}
\def\1{{\bf 1}}

\newcommand{\he}[2][2]{
	\FPeval{\resa}{#2^#1}%
	\edef\processme{\noexpand\cellcolor[gray]{\resa}}%
	\processme
	#2%
}

\newtheorem{theorem}{Theorem}

\newcommand{\twocell}[2][c]{%
  \begin{tabular}[#1]{@{}c@{}}#2\end{tabular}}
\usepackage{multirow}

\title{Effective Quantization Methods for Recurrent Neural Networks}

\author{Qinyao He, He Wen, Shuchang Zhou, Yuxin
Wu, Cong Yao, Xinyu Zhou, Yuheng Zou\\
  Megvii Inc. \\
\texttt{\{hqy, wenhe, zsc, wyx, yaocong, zxy,
zouyuheng\}@megvii.com}
\\
}

\begin{document}
\maketitle

\begin{abstract}
Reducing bit-widths of weights, activations, and gradients of a Neural Network
can shrink its storage size and memory usage, and also allow for faster training
and inference by exploiting bitwise operations.
However, previous attempts for quantization of RNNs show
considerable performance degradation when using low bit-width weights and
activations.
In this paper, we propose methods to quantize the
structure of gates and interlinks in LSTM and GRU cells. In addition, we propose
balanced quantization methods for weights to further reduce performance
degradation.
Experiments on PTB and IMDB datasets confirm effectiveness of our methods as
performances of our models match or surpass the previous state-of-the-art of
quantized RNN.
\end{abstract}

\section{Introduction}
Deep Neural Networks have become
important tools for modeling nonlinear functions in applications like computer
vision \citep{krizhevsky2012imagenet}, speech recognition \citep{hinton2012deep},
natural language processing \citep{bahdanau2014neural}, and computer games
\citep{silver2016mastering}.

However, inference and training of a DNN may involve up to billions of
operations for inputs likes images
\citep{krizhevsky2012imagenet,szegedy2014going}.
A DNN may also have large number of parameters, leading to large storage size and
runtime memory usage. Such intensive resource
requirements impede adoption of DNN in applications requiring real-time
responses, especially on resource-limited platforms. To alleviate these
requirements, many methods have been proposed, from both hardware and software
perspective \citep{farabet2011large,pham2012neuflow,chen2014diannao,chen2014dadiannao}.
For example, constraints may be imposed on the weights of DNN, like
sparsity \citep{han2015learning,han2015deep}, circulant
matrix \citep{cheng2015exploration}, low rank
\citep{jaderberg2014speeding,zhang2015accelerating}, vector quantization
\citep{gong2014compressing}, and hash trick \citep{chen2015compressing_hash} etc., to reduce the number of free parameters
and computation complexity. However, these methods use high bit-width numbers
for computations, which require availability of high precision multiply-and-add
instructions.

Another line of research tries to reduce bit-width of weights and activations of
a DNN by quantization to low bit-width numbers
\citep{rastegari2016xnor,hubara2016binarized,zhou2016dorefa,hubara2016quantized}.
Reducing bit-width of weights of a $32$-bit model to $k$ can shrink the storage
size of model to $\frac{k}{32}$ of the original size. Similarly, reducing
bit-widths of activations to $k$ can shrink the runtime memory usage by the same
proportion.
In addition, when the underlying platform supports efficient bitwise operations
and $bitcount$ that counts the number of bits in a bit vector, we can compute the
 inner product between bit vectors $\mathbf{x}$ $\mathbf{y}$ by the following
formula:
\begin{align}
\label{formula:bit-conv-kernel}
\mathbf{x} \cdot \mathbf{y} = 
\operatorname{bitcount}(\operatorname{and}(\mathbf{x}, \mathbf{y}))
\text{, } \forall i, x_i, y_i \in \{0, 1\}\text{.}
\end{align}
Consequently, convolutions between low bit-width numbers can be considerable
accelerated on platforms supporting efficient execution of bitwise operations, including CPU, GPU, FPGA and ASIC.
Previous works shows that using only 1-bit weights and 2-bit activation
can achieve 51\% top-1 accuracy on ImageNet datasets\citep{hubara2016quantized}.

However, in contrast to the extensive study in compression and quantization of
convolutional neural networks, little attention has been paid to reducing the
computational resource requirements of RNN. \citep{ott2016recurrent} claims
that the weight binarization method does not work with RNNs, and introduces
weight ternarization and leaves activations as floating point numbers.
\citep{hubara2016quantized} experiments with different combinations of
bit-widths for weights and activations, and shows 4-bit quantized CNN and RNN
can achieve comparable accuracy as their 32-bit counterpart. However, large
performance degradation occurs when quantizing weights and activations to 2-bit
numbers. Though \citep{hubara2016quantized} has their quantized CNN
open-sourced, neither of the two works open-source their quantized RNNs.

This paper makes the following contributions:
\begin{enumerate}
  \item We outline detailed design for quantizing two popular types of RNN
  cells:
  LSTM and GRU.
  We evaluate our model on different sets of bit-width configurations and two
  NLP tasks: Penn Treebank and IMDB.
  We demonstrate that by out design, quantization with 4-bit weights and
  activations can achieve almost the same performance to 32-bit. In addition, we
  have significantly better results when quantizing to lower bit-widths.
  \item We propose methods to quantize
  weights deterministically and adaptively to balanced distributions, especially
  when weights are 2-bits numbers. The balanced distribution of quantized
  weights leads to better utilization of the parameter space and consequently
  increases the prediction accuracy. We explicitly induce the balanced
  distribution by introducing parameter dependent thresholds into the
  quantization process during training.
  \item We release code for training our quantized RNNs online
  \footnote{https://github.com/hqythu/bit-rnn}. The code is implemented in
  TensorFlow\citep{abaditensorflow} framework.
\end{enumerate}

\section{Quantization Methods}
In this section we outline several quantization methods. W.l.o.g.,\ we assume
the input to the quantization is a matrix $\X$ unless otherwise specified.
When all entries of $\X$ are in close interval $[0, 1]$, we define the $k$-bit
uniform quantization $\operatorname{Q}_k$ as follows: .
\begin{align}
\operatorname{Q}_k(\X) = \frac{1}{2^k-1} \left \lfloor{(2^k-1)
\X+\frac12}\right \rfloor,
\notag\\
0\le x_{ij}\le 1 \forall i,j.
\end{align}

However, derivatives of this quantization function equals zero almost everywhere.
We adopt the ``straight-through
estimator'' (STE) method \citep{hinton2012neural,bengio2013estimating} to circumvent this problem.

For forward and backward pass of training neural network, using above quantization
method together with STE leads to the following update rule during forward and
backward propagation of neural networks:
\begin{align*}
\textbf{Forward: }& q \leftarrow \operatorname{Q}_k(p) \\
\textbf{Backward: }& \frac{\partial{c}}{\partial{p}} \leftarrow \frac{\partial{c}}{\partial{q}} \text{.}
\end{align*}

\subsection{Deterministic Quantization}
When entries in $\X$ are not constrained in closed interval $[0, 1]$, an affine
transform need to be applied before using function $\operatorname{Q}_k$. A
straightforward transformation can be done using minimum and maximum of $\X$ to
get $\tilde{\X}$, the standardized version of $\X$:
\begin{align*}
\tilde{\X} = \frac{\X - \beta}{\alpha} \\
\alpha = \max(\X)-\min(\X) \\
\beta = \min(\X)
\end{align*}

After quantization, we can apply a reverse affine transform to approximate the
original values. Overall, the quantized result is:
\begin{align*}
\operatorname{Q}_k^{det}(\X) = \alpha \operatorname{Q}_k(\tilde{\X}) + \beta
\approx \X
\end{align*}

\subsection{Balanced Deterministic Quantization}
When we quantize values, it may be desirable to make the
quantized values have balanced distributions, so as to take full advantage of
the available parameter space.
Ordinarily, this is not possible as the distribution of the input values has
already been fixed. In particular, using $\operatorname{Q}_k^{det}$ do not
exert any impacts on the distribution of quantized values.

Next we show that we can induce more uniform distributions of quantized values,
by introducing parameter dependent adaptive thresholds $\gamma
\operatorname{median}(|\X|)$. We first introduce a different standardization
transform that produces $\hat{\X}$, and define a balanced quantization method $\hat{\operatorname{Q}}_{k}^{bal}$ as follows:
\begin{align}
\hat{\X} = \mathrm{clip}(\frac{\X}{\gamma \operatorname{median}(|\X|)},
-\frac12, \frac12) + \frac12 \label{lab:bal-quantize-median}\\
\hat{\operatorname{Q}}_{k}^{bal}(\X) = \alpha \operatorname{Q}_k(\hat{\X}) +
\beta\notag
\end{align}

The only difference between
$\hat{\operatorname{Q}}_k^{bal}$ and $\operatorname{Q}_k^{det}$ lies in
difference of standardization. In fact, when the extremal values of $\X$ are symmetric
around zero, i.e.
\begin{align*}
\operatorname{min}(\X) + \operatorname{max}(\X)=0,
\end{align*}
we may rewrite
$\operatorname{Q}_k^{det}$ equivalently as follows to make the similarity
between $\operatorname{Q}_k^{bal}$ and $\operatorname{Q}_k^{det}$ more obvious:
\begin{align*}
\tilde{\X} & = \frac{\X - \operatorname{min}(\X)}{\operatorname{max}(\X)
- \operatorname{min}(\X)} \\
& = \frac{\X}{2 \operatorname{max}(\X)} + \frac12 \\
& = \operatorname{clip}(\frac{\X}{2 \operatorname{max}(\X)}, -\frac12, 
\frac12) + \frac12 \\
\operatorname{Q}_k^{det}(\X) & = \alpha \operatorname{Q}_k(\tilde{\X}) +
\beta
\end{align*}

Hence the only difference between $\hat{\operatorname{Q}}_k^{bal}$ and 
$\operatorname{Q}_k^{det}$ lies in difference between properties of
$2\operatorname{max}(\X)$ and $\gamma \operatorname{median}(|\X|)$. We
find that as $median$ is an order statistics, using it as threshold will produce
an auto-balancing effect.

\subsubsection{The Auto-balancing effect of $\hat{\operatorname{Q}}_k^{bal}$}
We consider the case when bit-width is $2$ as an example. In this case, under the
symmetric distribution assumption, we can prove the auto-balancing effect of
$\hat{\operatorname{Q}}_k^{bal}$.
\begin{theorem}\label{thm:balanced-quantization}
If $k=2, \gamma=3$, and suppose $\X$ are symmetrically distributed around zero
and there are no two entries in $\X$ that are equal, then the four bars in the
histogram of $\hat{\operatorname{Q}}_k^{bal}(\X)$ will all have exactly the
same height.
\end{theorem}
\begin{proof}
By Formula~\ref{lab:bal-quantize-median},
entries of $\hat{\operatorname{Q}}_k^{bal}(\X)$ will be
equal to 1 if corresponding entries in $\X$ are above
$\frac{\gamma}3\operatorname{median}(|\X|)$, equal to $\frac23$ if between
$0$ and $\frac{\gamma}3\operatorname{median}(|\X|)$, equal to $\frac13$ if
between $-\frac{\gamma}3\operatorname{median}(|\X|)$ and $0$, and equal to $0$
if below $-\frac{\gamma}3\operatorname{median}(|\X|)$.
When $\gamma=3$ and $\X$ are symmetrically distributed around zero, the
values in $\X$ will be thresholded by $-\operatorname{median}(|\X|)$, $0$, and
$\operatorname{median}(|\X|)$ into four bins. By the property of median, and the
symmetric distribution assumption, the four bins will contain the same number
of quantized values.
\end{proof}

In practice, computing $\operatorname{median}(|\X|)$ may not be computationally
convenient as it requires sorting.
We note that when a distribution has bounded variance $\sigma$, the mean $\mu$
approximates the median $m$ as there is an inequality bounding the difference\citep{mallows1991another}:
\begin{align*}
|\mu - m| \le \sigma.
\end{align*}
Hence we may use $\operatorname{mean}(|\X|)$ instead of
$\operatorname{median}(|\X|)$ in the quantization. Though with error introduced,
empirically we can still observed nearly-balanced distribution.

If we further assume the weights follow zero-mean normal distribution $\mathcal
{N}(0 ,\,\sigma ^{2})$, then $|\X|$ follows half-normal distribution. By simple calculations we have:
\begin{align*}
\frac{\operatorname{mean}(|\X|)}{\operatorname{median}(|\X|)}
= \frac{{\frac  {\sigma {\sqrt  {2}}}{{\sqrt  {\pi }}}}}{{\displaystyle \sigma
{\sqrt {2}}\operatorname {erf} ^{-1}(\frac12)}} \approx
\frac{1}{0.4769\sqrt{\pi}} \approx 1.1830
\end{align*}
and
\begin{align*}
3\operatorname{median}(|\X|) \approx 2.5359\operatorname{mean}(|\X|)
\end{align*}



Putting all these things
together we have the balanced deterministic quantization method:
\begin{align}
\hat{\X} = \operatorname{clip}(\frac{\X}{\gamma \operatorname{mean}(|\X|)}, -\frac12, 
\frac12) + \frac12 \label{lab:bal-quantize}\\
\operatorname{Q}_{k}^{bal}(\X) = \alpha \operatorname{Q}_k(\hat{\X}) +
\beta \approx \X \notag,
\end{align}
where a natural choice of $\gamma$ would be 3 or 2.5 (rounding 2.5359 to a
short binary number) under different assumptions.
In our following experiments, we adopt 2.5 as the scaling factor.

Although the above argument for balanced quantization applies only to 2-bit
quantization, we argue more bit-width also benefit from avoiding extreme value
from extending the value range thus increase rounding error. It should be
noted that for 1-bit quantization (binarization), the scaling factor should be
$2\operatorname{mean}(|\X|)$, which can be proved to be optimal in the sense
of reconstruction error measured by Frobenius norm, as in
\citep{rastegari2016xnor}. However, the proof relies on the constant norm
property of 1-bit representations, and does not generalize to the cases of
other bit-widths.

\subsection{Quantization of Weights}

Weights in neural networks are sometimes known to have a bell-style distribution
around zero, similar to normal distribution. Hence we can assume $\X$ to have
symmetric distribution around 0, and apply the above equation for balanced
quantization as
\begin{align*}
scale = \operatorname{mean}(\operatorname{abs}(X)) * 2.5 \\
\operatorname{Q}_k^{bal}(\X) = \operatorname{Q}_k(\frac{\X}{scale}) * scale
\approx \X
\text{.}
\end{align*}
To include the quantization into the computation graph of a neural network, we
apply STE on entire expression rather than only $\operatorname{Q}_k$ itself.

\begin{align*}
\textbf{Forward: }& q \leftarrow \operatorname{Q}_k^{bal}(p) \\
\textbf{Backward: }& \frac{\partial{c}}{\partial{p}} \leftarrow \frac{\partial{c}}{\partial{q}} \text{.}
\end{align*}

The specialty about the balanced quantization method
$\operatorname{Q}_{k}^{bal}$ is that in general, it distort the extremal values
due to the clipping in Formula~\ref{lab:bal-quantize}, which in general
contribute more to the computed sums of inner products.
However, in case where the values to be quantized are weights of neural networks
and if we introduce the balanced quantization into the training process, we conjecture
that the neural networks may gradually adapt to the distortions, so that
distributions of weights may be induced to be more balanced. The more balanced
distribution will increase the effective bit-width of neural networks, leading
to better prediction accuracy.
We will empirically validate this conjecture through experiments in
Section~\ref{sec:exp}.

\subsection{Quantization of Activations}
Quantization of activation follows the method in \cite{zhou2016dorefa}, assuming
output of the previous layer has passed through a bounded activation function
$h$, and we will apply quantization directly to them. In fact, we find that
adding a scaling term containing mean or max value to the activations may harm
prediction accuracy.

There is a design choice on what range of quantized value should be. One choice
is symmetric distribution around 0. 
Under this choice, inputs are bounded by activation
function to $[-0.5, 0.5]$, and then shifted to the right by 0.5 before feeding into
$\operatorname{Q}_k$ and then shift back.
$$
\X_q = \operatorname{Q}_k(\X + 0.5) - 0.5
$$

Another choice is having value range of $[0, 1]$, which is closer to the value
range of ReLU activation function. Under this choice, we can directly apply
$\operatorname{Q}_k$. 
For commonly used $\tanh$ activation with domain $[-1, 1]$ in RNNs, it seems natural to use symmetry
quantization. However, we will point out some considerations for using
quantization to range $[0, 1]$ in Section~\ref{sec:quant-rnn}.

\section{Quantization of Recurrent Neural Network}
\label{sec:quant-rnn}

In this section, we detail our design considerations for quantization of
recurrent neural networks.
Different from plain feed forward neural network, recurrent neural networks,
especially Long Short Term Memory \citep{hochreiter1997long} and Gated
Recurrent Unit \citep{chung2014empirical}, have subtle and delicately
designed structure, which makes their quantization more complex and need more
careful considerations. Nevertheless, the major algorithm is the same as
Algorithm~1 in \cite{zhou2016dorefa}.

\subsection{Dropout}
It is well known that as fully-connected layers have large number of parameters,
they are prone to overfit \citep{srivastava2014dropout}. There are several
FC-like structures in a RNN, for example the input, output and transition
matrices in RNN cells (like GRU and LSTM) and the final FC layer
for softmax classification. The dropout technique, which randomly dropping a portion
of features to 0 at training time, turns out be also an
effective way of alleviating overfitting in RNN \citep{zaremba2014recurrent}.

As dropped activations are zero, it is necessary to have zero values
in the range of quantized values. For symmetric quantization to range $[-0.5,
0.5]$, $0$ does not exist in range of $\operatorname{Q}_k(X+ 0.5) - 0.5$.
Hence we use $[0, 1]$ as the range of quantized values when
dropout is needed.

\subsection{Embedding Layer}
In tasks related to Natural Language Processing, the input words which
are represented by ID's, are embedded into a low-dimensional space before
feeding into RNNs. The word embedding
matrix is in $\mathbb{R}^{|V|\times N}$, where $|V|$ is the size of vocabulary
and $N$ is length of embedded vectors.

Quantization of weights in embedding layers turns out to be different from
quantization of weights in FC layers.
In fact, the weights of embedding layers actually behave like activations: a
certain row is selected and fed to the next layer, so the quantization method
should be the same as that of activations rather than that of weights.
Similarly, as dropout may also be applied on the outputs of embedding layers, it
is necessary to bound the values in embedding matrices to $[0,1]$.

To clip the value range of weights of embedding layers, a natural choice would
be using sigmoid function $h(x) = \frac{1}{1 + e^{-x}}$ such that $h(\W)$ will
be used as parameters of embedding layers, but we observe severe vanishing
gradient problem for gradients $\frac{\partial \text{Cost}}{\partial\W}$ in
training process.
Hence instead, we directly apply a $\operatorname{clip}$ function $\max(\min(\W,
1), 0)$, and random initialize the embedding matrices with values drawn
from uniform distribution $\mathrm{U}(0,1)$. These two measures are found to
improve performance of the model.

\subsection{Quantization of GRU}
We first investigate quantization of GRU as it is structurally simpler.
The basic structure of GRU cell may be described as follows:
\begin{align*}
z_t &= \sigma (W_z \cdot [h_{t-1}, x_t]) \\
r_t &= \sigma (W_r \cdot [h_{t-1}, x_t]) \\
\widetilde{h_t} &= \tanh (W \cdot [r_t * h_{t-1}, x_t]) \\
h_t &= (1 - z_t) * h_{t-1} + z_t * \widetilde{h_t}
\text{,}
\end{align*}
where $\sigma$ stands for the sigmoid function.

Recall that to benefit from the speed advantage of bit convolution kernels, we
need to make the two matrix inputs for multiply in low bit form, so that the dot product can be calculated
by bitwise operation.
For plain feed forward neural networks, as the convolutions take up most of
computation time, we can get decent acceleration by quantization of inputs of
convolutions and their weights.
But when it comes to more complex structures like GRU, we need to check the
bit-width of each interlink.
 
Except for matrix multiplications needed to compute $z_t$,$r_t$ and
$\widetilde{h_t}$, the \emph{gate} structure of $\widetilde{h_t}$ and $h_t$
brings in the need for element-wise multiplication. As the output of
the sigmoid function may have large bit-width, the element-wise multiplication
may need be done in floating point numbers (or in higher fixed-point format). As $\widetilde{h_t}$ and $h_t$ are
also the inputs to computations at the next timestamp, and noting that a
quantized value multiplied by a quantized value will have a larger bit-width, we need to
insert additional quantization steps after element-wise multiplications.




Another problem with quantization of GRU structure lies in the different value
range of gates. The range of $\tanh$ is $[-1, 1]$, which is different from the
value range $[0,1]$ of $z_t$ and $r_t$. If we want to preserve the original
activation functions, we will have the following quantization scheme:
\begin{align*}
z_t &= \sigma (W_z \cdot [h_{t-1}, x_t]) \\
r_t &= \sigma (W_r \cdot [h_{t-1}, x_t]) \\
\widetilde{h_t} &= \tanh (W \cdot [2\operatorname{Q}_k(\frac12(r_t *
h_{t-1})+\frac12)-1, x_t])
\\
h_t &=2\operatorname{Q}_k(\frac12((1 - z_t) * h_{t-1} + z_t *
\widetilde{h_t})+\frac12)-1
\text{,}
\end{align*}
where we assume the weights $W_z, W_r, W$ have already been
quantized to $[-1, 1]$, and input $x_t$ have already been
quantized to $[-1, 1]$.

However, we note that the quantization function already has an affine
transform to shift the value range. To simplify the implementation, we replace
the activation functions of $\widetilde{h_t}$ to be the sigmoid function, so
that $(1 - z_t) * h_{t-1} + z_t * \widetilde{h_t} \in [0,1]$.

Summarizing the above considerations, the quantized version of GRU could be
written as
\begin{align*}
z_t &= \sigma (W_z \cdot [h_{t-1}, x_t]) \\
r_t &= \sigma (W_r \cdot [h_{t-1}, x_t]) \\
\widetilde{h_t} &= \sigma (W \cdot [\operatorname{Q}_k(r_t * h_{t-1}), x_t]) \\
h_t &=\operatorname{Q}_k((1 - z_t) * h_{t-1} + z_t * \widetilde{h_t})
\text{,}
\end{align*}
where we assume the weights $W_z, W_r, W$ have already been
quantized to $[-1, 1]$, and input $x_t$ have already been
quantized to $[0, 1]$.

\subsection{Quantization of LSTM}
The structure of LSTM can be described as follows:
\begin{align*}
f_t &= \sigma (W_f \cdot [h_{t-1}, x_t] + b_f) \\
i_t &= \sigma (W_i \cdot [h_{t-1}, x_t] + b_i) \\
\widetilde{C_t} &= \tanh (W_C \cdot [h_{t-1}, x_t] + b_i) \\
C_t &= f_t * C_{t-1} + i_t * \widetilde{C_t} \\
o_t &= \sigma (W_o \cdot [h_{t-1}, x_t] + b_o) \\
h_t &= o_t * \tanh (C_t)
\end{align*}

Different from GRU, $C_t$ can not be easily quantized, since the value is unbounded by not
using activation function like $\tanh$ and the sigmoid function. This difficulty
comes from structure design and can not be alleviated without introducing extra facility to clip value
ranges.
But it can be noted that the computations involving $C_t$ are all
element-wise multiplications and additions, which may take much less time than
computing matrix products.
For this reason, we leave $C_t$ to be in floating point form. 

To simplify implementation, $\tanh$ activation for output may be changed to
the sigmoid function.

Summarizing above changes, the formula for quantized LSTM can be:
\begin{align*}
f_t &= \sigma (W_f \cdot [h_{t-1}, x_t] + b_f) \\
i_t &= \sigma (W_i \cdot [h_{t-1}, x_t] + b_i) \\
\widetilde{C_t} &= \tanh (W_C \cdot [h_{t-1}, x_t] + b_i) \\
C_t &= f_t * C_{t-1} + i_t * \widetilde{C_t} \\
o_t &= \sigma (W_o \cdot [h_{t-1}, x_t] + b_o) \\
h_t &= \operatorname{Q}_k(o_t * \sigma (C_t))
\text{,}
\end{align*}
where we assume the weights $W_f, W_i, W_C, W_o$ have already been
quantized to $[-1, 1]$, and input $x_t$ have already been
quantized to $[0, 1]$.

\section{Experiment Results}
\label{sec:exp}

We evaluate the quantized RNN models on two tasks: language modeling and
sentence classification.

\subsection{Experiments on Penn Treebank dataset}
For language modeling we use Penn Treebank dataset \citep{taylor2003penn}, which
contains 10K unique words.
We download the data from Tomas Mikolov's
webpage\footnote{http://www.fit.vutbr.cz/~imikolov/rnnlm/simple-examples.tgz}.
For fair comparison, in the following experiments, our model all use one hidden
layer with 300 hidden units, which is the same setting as
\cite{hubara2016quantized}. A word embedding layer is used at the input side of
the network whose weights are trained from scratch.
The performance is measured in perplexity per word (PPW) metric.

During experiments we find the magnitudes of values in dense matrices or full
connected layers explode when using small bit-width, and result in overfitting and divergence. This can be 
alleviated by adding $\tanh$ to constrain the value ranges or adding weight
decays for regularization.

\begin{table}[!ht] \centering \small
	\begin{center}
		\begin{tabular}{
				 p{0.2\linewidth}  p{0.15\linewidth}  p{0.15\linewidth} 
				p{0.15\linewidth}  p{0.15\linewidth}  }
			\hline
			\multirow{2}{*}{\textbf{Model}} & \multirow{2}{*}{weight-bits} & \multirow{2}{*}{activation-bits} & \multicolumn{2}{c}{PPW} \\
			\cline{4-5}
			&  &  & balanced & unbalanced \\
			\hline
			GRU & 1 & 2 & 285 & diverge \\
			GRU & 1 & 32 & 178 & diverge \\
			GRU & 2 & 2 & 150 & 165 \\
			GRU & 2 & 3 & 128 & 141 \\
			GRU & 3 & 3 & 109 & 110 \\
			GRU & 4 & 4 & 104 & 102 \\
			GRU & 32 & 32 & - & 100 \\
			\hline
			LSTM & 1 & 2 & 257 & diverge \\
			LSTM & 1 & 32 & 198 & diverge \\
			LSTM & 2 & 2 & 152 & 164 \\
			LSTM & 2 & 3 & 142 & 155 \\
			LSTM & 3 & 3 & 120 & 122 \\
			LSTM & 4 & 4 & 114 & 114 \\
			LSTM & 32 & 32 & - & 109 \\
			\hline
			\twocell{LSTM \\ \citep{hubara2016quantized}} & 2 & 3 &  & 220 \\			
			\twocell{LSTM \\ \citep{hubara2016quantized}} & 4 & 4 &  & 100 \\						
			\hline
		\end{tabular}
	\end{center}
	\caption{Quantized RNNs on PTB datasets}
	\label{tab:bit_rnn_ptb}
\end{table}

Our result is in agreement with \citep{hubara2016quantized} where they claim
using 4-bit weights and activations can achieve almost the same performance as
32-bit.
However, we report higher accuracy when using less bits, such as 2-bit weight
and activations. The 2-bit weights and 3-bit activations LSTM achieve 146 PPW,
which outperforms the counterpart in \citep{hubara2016quantized} by a large
margin.

We also perform experiments in which weights are binarized. The models can
converge, though with large performance degradations.


\subsection{Experiments on Penn IMDB datasets}
We do further experiments on sentence classification using IMDB datasets
\citep{maas-EtAl:2011:ACL-HLT2011}.
We pad or cut each sentence to 500 words, word embedding vectors of
length 512, and a single recurrent layer with 512 number of hidden neurons. All
models are trained using ADAM\citep{kingma2014adam} learning rule with learning
rate $10^{-3}$.

\begin{table}[!ht] \centering \small
	\begin{center}
		\begin{tabular}{
				 p{0.1\linewidth}  p{0.15\linewidth}  p{0.15\linewidth}  p{0.15\linewidth}  p{0.15\linewidth}  }
			\hline
			\multirow{2}{*}{\textbf{Model}} & \multirow{2}{*}{weight-bits} & \multirow{2}{*}{activation-bits} & \multicolumn{2}{c}{Accuracy} \\
			\cline{4-5}
			&  &  & balanced & unbalanced \\
			\hline
			GRU & 1 & 2 & 0.8684 & diverge \\
			GRU & 2 & 2 & 0.8708 & 0.86056 \\
			GRU & 4 & 4 & 0.88132 & 0.88248 \\
			GRU & 32 & 32 & - & 0.90537 \\
			\hline
			LSTM & 1 & 2 & 0.87888 & diverge \\
			LSTM & 2 & 2 & 0.8812 & 0.83971 \\
			LSTM & 4 & 4 & 0.88476 & 0.86788 \\
			LSTM & 32 & 32 & - & 0.89541 \\
			\hline
		\end{tabular}
	\end{center}
	\caption{Quantized RNNs on IMDB sentence classification}
	\label{tab:bit_rnn_imdb}
\end{table}

As IMDB is a fairly simple dataset, we observe little performance
degradation even when quantizing to 1-bit weights and 2-bit activations.

\subsection{Effects of Balanced Distribution}
All the above experiments show balanced quantization leads to better results
compared to unbalanced counterparts, especially when quantizing to 2-bit
weights.
However, for 4-bit weights,
there is no clear gap between scaling by mean and scaling by max (i.e. balanced
and unbalanced quantization), indicating that more effective methods for
quantizing to 4-bit need to be discovered.

\section{Conclusion and Future Work}
We have proposed methods for effective quantization of RNNs. By using carefully
designed structure and a balanced quantization methods, we have matched
or surpassed previous state-of-the-arts in prediction accuracy, especially when
quantizing to 2-bit weights.

The balanced quantization method for weights we propose can
induce balanced distribution of quantized weight value to maximum the
utilization of parameter space. The method may also be applied to quantization
of CNNs.

As future work, first, the method to induce balanced
weight quantization when bit-width is more than 2 remains to be found. Second,
we have observed some difficulties for quantizing the cell paths in LSTM, which produces unbounded
values.
One possible way to address this problem is introducing novel scaling schemes to
quantize the activations that can deal with unbounded values.
Finally, as we have observed GRU and LSTM have different properties
in quantization, it remains to be shown whether there exists more efficient
recurrent structures designed specifically to facilitate quantization.

\bibliographystyle{iclr2017_conference}
\bibliography{thesis}

\begin{thebibliography}{32}
\providecommand{\natexlab}[1]{#1}
\providecommand{\url}[1]{\texttt{#1}}
\expandafter\ifx\csname urlstyle\endcsname\relax
  \providecommand{\doi}[1]{doi: #1}\else
  \providecommand{\doi}{doi: \begingroup \urlstyle{rm}\Url}\fi

\bibitem[Abadi et~al.()Abadi, Agarwal, Barham, Brevdo, Chen, Citro, Corrado,
  Davis, Dean, Devin, et~al.]{abaditensorflow}
Mart{\i}n Abadi, Ashish Agarwal, Paul Barham, Eugene Brevdo, Zhifeng Chen,
  Craig Citro, Greg~S Corrado, Andy Davis, Jeffrey Dean, Matthieu Devin, et~al.
\newblock Tensorflow: Large-scale machine learning on heterogeneous systems,
  2015.
\newblock \emph{Software available from tensorflow. org}.

\bibitem[Bahdanau et~al.(2014)Bahdanau, Cho, and Bengio]{bahdanau2014neural}
Dzmitry Bahdanau, Kyunghyun Cho, and Yoshua Bengio.
\newblock Neural machine translation by jointly learning to align and
  translate.
\newblock \emph{arXiv preprint arXiv:1409.0473}, 2014.

\bibitem[Bengio et~al.(2013)Bengio, L{\'e}onard, and
  Courville]{bengio2013estimating}
Yoshua Bengio, Nicholas L{\'e}onard, and Aaron Courville.
\newblock Estimating or propagating gradients through stochastic neurons for
  conditional computation.
\newblock \emph{arXiv preprint arXiv:1308.3432}, 2013.

\bibitem[Chen et~al.(2014{\natexlab{a}})Chen, Du, Sun, Wang, Wu, Chen, and
  Temam]{chen2014diannao}
Tianshi Chen, Zidong Du, Ninghui Sun, Jia Wang, Chengyong Wu, Yunji Chen, and
  Olivier Temam.
\newblock Diannao: A small-footprint high-throughput accelerator for ubiquitous
  machine-learning.
\newblock In \emph{ACM Sigplan Notices}, volume~49, pp.\  269--284. ACM,
  2014{\natexlab{a}}.

\bibitem[Chen et~al.(2015)Chen, Wilson, Tyree, Weinberger, and
  Chen]{chen2015compressing_hash}
Wenlin Chen, James~T Wilson, Stephen Tyree, Kilian~Q Weinberger, and Yixin
  Chen.
\newblock Compressing neural networks with the hashing trick.
\newblock \emph{arXiv preprint arXiv:1504.04788}, 2015.

\bibitem[Chen et~al.(2014{\natexlab{b}})Chen, Luo, Liu, Zhang, He, Wang, Li,
  Chen, Xu, Sun, et~al.]{chen2014dadiannao}
Yunji Chen, Tao Luo, Shaoli Liu, Shijin Zhang, Liqiang He, Jia Wang, Ling Li,
  Tianshi Chen, Zhiwei Xu, Ninghui Sun, et~al.
\newblock Dadiannao: A machine-learning supercomputer.
\newblock In \emph{Proceedings of the 47th Annual IEEE/ACM International
  Symposium on Microarchitecture}, pp.\  609--622. IEEE Computer Society,
  2014{\natexlab{b}}.

\bibitem[Cheng et~al.(2015)Cheng, Yu, Feris, Kumar, Choudhary, and
  Chang]{cheng2015exploration}
Yu~Cheng, Felix~X Yu, Rogerio~S Feris, Sanjiv Kumar, Alok Choudhary, and Shi-Fu
  Chang.
\newblock An exploration of parameter redundancy in deep networks with
  circulant projections.
\newblock In \emph{Proceedings of the IEEE International Conference on Computer
  Vision}, pp.\  2857--2865, 2015.

\bibitem[Chung et~al.(2014)Chung, Gulcehre, Cho, and
  Bengio]{chung2014empirical}
Junyoung Chung, Caglar Gulcehre, KyungHyun Cho, and Yoshua Bengio.
\newblock Empirical evaluation of gated recurrent neural networks on sequence
  modeling.
\newblock \emph{arXiv preprint arXiv:1412.3555}, 2014.

\bibitem[Farabet et~al.(2011)Farabet, LeCun, Kavukcuoglu, Culurciello, Martini,
  Akselrod, and Talay]{farabet2011large}
Cl{\'e}ment Farabet, Yann LeCun, Koray Kavukcuoglu, Eugenio Culurciello, Berin
  Martini, Polina Akselrod, and Selcuk Talay.
\newblock Large-scale fpga-based convolutional networks.
\newblock \emph{Scaling up Machine Learning: Parallel and Distributed
  Approaches}, pp.\  399--419, 2011.

\bibitem[Gong et~al.(2014)Gong, Liu, Yang, and Bourdev]{gong2014compressing}
Yunchao Gong, Liu Liu, Ming Yang, and Lubomir Bourdev.
\newblock Compressing deep convolutional networks using vector quantization.
\newblock \emph{arXiv preprint arXiv:1412.6115}, 2014.

\bibitem[Han et~al.(2015{\natexlab{a}})Han, Mao, and Dally]{han2015deep}
Song Han, Huizi Mao, and William~J Dally.
\newblock Deep compression: Compressing deep neural networks with pruning,
  trained quantization and huffman coding.
\newblock \emph{arXiv preprint arXiv:1510.00149}, 2015{\natexlab{a}}.

\bibitem[Han et~al.(2015{\natexlab{b}})Han, Pool, Tran, and
  Dally]{han2015learning}
Song Han, Jeff Pool, John Tran, and William Dally.
\newblock Learning both weights and connections for efficient neural network.
\newblock In \emph{Advances in Neural Information Processing Systems}, pp.\
  1135--1143, 2015{\natexlab{b}}.

\bibitem[Hinton et~al.(2012{\natexlab{a}})Hinton, Deng, Yu, Dahl, Mohamed,
  Jaitly, Senior, Vanhoucke, Nguyen, Sainath, et~al.]{hinton2012deep}
Geoffrey Hinton, Li~Deng, Dong Yu, George~E Dahl, Abdel-rahman Mohamed, Navdeep
  Jaitly, Andrew Senior, Vincent Vanhoucke, Patrick Nguyen, Tara~N Sainath,
  et~al.
\newblock Deep neural networks for acoustic modeling in speech recognition: The
  shared views of four research groups.
\newblock \emph{Signal Processing Magazine, IEEE}, 29\penalty0 (6):\penalty0
  82--97, 2012{\natexlab{a}}.

\bibitem[Hinton et~al.(2012{\natexlab{b}})Hinton, Srivastava, and
  Swersky]{hinton2012neural}
Geoffrey Hinton, Nitsh Srivastava, and Kevin Swersky.
\newblock Neural networks for machine learning.
\newblock \emph{Coursera, video lectures}, 264, 2012{\natexlab{b}}.

\bibitem[Hochreiter \& Schmidhuber(1997)Hochreiter and
  Schmidhuber]{hochreiter1997long}
Sepp Hochreiter and J{\"u}rgen Schmidhuber.
\newblock Long short-term memory.
\newblock \emph{Neural computation}, 9\penalty0 (8):\penalty0 1735--1780, 1997.

\bibitem[Hubara et~al.(2016{\natexlab{a}})Hubara, Courbariaux, Soudry,
  El-Yaniv, and Bengio]{hubara2016quantized}
Itay Hubara, Matthieu Courbariaux, Daniel Soudry, Ran El-Yaniv, and Yoshua
  Bengio.
\newblock Quantized neural networks: Training neural networks with low
  precision weights and activations.
\newblock \emph{arXiv preprint arXiv:1609.07061}, 2016{\natexlab{a}}.

\bibitem[Hubara et~al.(2016{\natexlab{b}})Hubara, Soudry, and
  Yaniv]{hubara2016binarized}
Itay Hubara, Daniel Soudry, and Ran~El Yaniv.
\newblock Binarized neural networks.
\newblock \emph{arXiv preprint arXiv:1602.02505}, 2016{\natexlab{b}}.

\bibitem[Jaderberg et~al.(2014)Jaderberg, Vedaldi, and
  Zisserman]{jaderberg2014speeding}
Max Jaderberg, Andrea Vedaldi, and Andrew Zisserman.
\newblock Speeding up convolutional neural networks with low rank expansions.
\newblock \emph{arXiv preprint arXiv:1405.3866}, 2014.

\bibitem[Kingma \& Ba(2014)Kingma and Ba]{kingma2014adam}
Diederik Kingma and Jimmy Ba.
\newblock Adam: A method for stochastic optimization.
\newblock \emph{arXiv preprint arXiv:1412.6980}, 2014.

\bibitem[Krizhevsky et~al.(2012)Krizhevsky, Sutskever, and
  Hinton]{krizhevsky2012imagenet}
Alex Krizhevsky, Ilya Sutskever, and Geoffrey~E Hinton.
\newblock Imagenet classification with deep convolutional neural networks.
\newblock In \emph{Advances in neural information processing systems}, pp.\
  1097--1105, 2012.

\bibitem[Maas et~al.(2011)Maas, Daly, Pham, Huang, Ng, and
  Potts]{maas-EtAl:2011:ACL-HLT2011}
Andrew~L. Maas, Raymond~E. Daly, Peter~T. Pham, Dan Huang, Andrew~Y. Ng, and
  Christopher Potts.
\newblock Learning word vectors for sentiment analysis.
\newblock In \emph{Proceedings of the 49th Annual Meeting of the Association
  for Computational Linguistics: Human Language Technologies}, pp.\  142--150,
  Portland, Oregon, USA, June 2011. Association for Computational Linguistics.
\newblock URL \url{http://www.aclweb.org/anthology/P11-1015}.

\bibitem[Mallows(1991)]{mallows1991another}
Colin Mallows.
\newblock Another comment on o’cinneide.
\newblock \emph{The American Statistician}, 45\penalty0 (3):\penalty0 257,
  1991.

\bibitem[Ott et~al.(2016)Ott, Lin, Zhang, Liu, and Bengio]{ott2016recurrent}
Joachim Ott, Zhouhan Lin, Ying Zhang, Shih-Chii Liu, and Yoshua Bengio.
\newblock Recurrent neural networks with limited numerical precision.
\newblock \emph{arXiv preprint arXiv:1608.06902}, 2016.

\bibitem[Pham et~al.(2012)Pham, Jelaca, Farabet, Martini, LeCun, and
  Culurciello]{pham2012neuflow}
Phi-Hung Pham, Darko Jelaca, Clement Farabet, Berin Martini, Yann LeCun, and
  Eugenio Culurciello.
\newblock Neuflow: Dataflow vision processing system-on-a-chip.
\newblock In \emph{Circuits and Systems (MWSCAS), 2012 IEEE 55th International
  Midwest Symposium on}, pp.\  1044--1047. IEEE, 2012.

\bibitem[Rastegari et~al.(2016)Rastegari, Ordonez, Redmon, and
  Farhadi]{rastegari2016xnor}
Mohammad Rastegari, Vicente Ordonez, Joseph Redmon, and Ali Farhadi.
\newblock Xnor-net: Imagenet classification using binary convolutional neural
  networks.
\newblock \emph{arXiv preprint arXiv:1603.05279}, 2016.

\bibitem[Silver et~al.(2016)Silver, Huang, Maddison, Guez, Sifre, Van
  Den~Driessche, Schrittwieser, Antonoglou, Panneershelvam, Lanctot,
  et~al.]{silver2016mastering}
David Silver, Aja Huang, Chris~J Maddison, Arthur Guez, Laurent Sifre, George
  Van Den~Driessche, Julian Schrittwieser, Ioannis Antonoglou, Veda
  Panneershelvam, Marc Lanctot, et~al.
\newblock Mastering the game of go with deep neural networks and tree search.
\newblock \emph{Nature}, 529\penalty0 (7587):\penalty0 484--489, 2016.

\bibitem[Srivastava et~al.(2014)Srivastava, Hinton, Krizhevsky, Sutskever, and
  Salakhutdinov]{srivastava2014dropout}
Nitish Srivastava, Geoffrey~E Hinton, Alex Krizhevsky, Ilya Sutskever, and
  Ruslan Salakhutdinov.
\newblock Dropout: a simple way to prevent neural networks from overfitting.
\newblock \emph{Journal of Machine Learning Research}, 15\penalty0
  (1):\penalty0 1929--1958, 2014.

\bibitem[Szegedy et~al.(2014)Szegedy, Liu, Jia, Sermanet, Reed, Anguelov,
  Erhan, Vanhoucke, and Rabinovich]{szegedy2014going}
Christian Szegedy, Wei Liu, Yangqing Jia, Pierre Sermanet, Scott Reed, Dragomir
  Anguelov, Dumitru Erhan, Vincent Vanhoucke, and Andrew Rabinovich.
\newblock Going deeper with convolutions.
\newblock \emph{arXiv preprint arXiv:1409.4842}, 2014.

\bibitem[Taylor et~al.(2003)Taylor, Marcus, and Santorini]{taylor2003penn}
Ann Taylor, Mitchell Marcus, and Beatrice Santorini.
\newblock The penn treebank: an overview.
\newblock In \emph{Treebanks}, pp.\  5--22. Springer, 2003.

\bibitem[Zaremba et~al.(2014)Zaremba, Sutskever, and
  Vinyals]{zaremba2014recurrent}
Wojciech Zaremba, Ilya Sutskever, and Oriol Vinyals.
\newblock Recurrent neural network regularization.
\newblock \emph{arXiv preprint arXiv:1409.2329}, 2014.

\bibitem[Zhang et~al.(2015)Zhang, Zou, He, and Sun]{zhang2015accelerating}
Xiangyu Zhang, Jianhua Zou, Kaiming He, and Jian Sun.
\newblock Accelerating very deep convolutional networks for classification and
  detection.
\newblock \emph{IEEE transactions on pattern analysis and machine
  intelligence}, 2015.

\bibitem[Zhou et~al.(2016)Zhou, Wu, Ni, Zhou, Wen, and Zou]{zhou2016dorefa}
Shuchang Zhou, Yuxin Wu, Zekun Ni, Xinyu Zhou, He~Wen, and Yuheng Zou.
\newblock Dorefa-net: Training low bitwidth convolutional neural networks with
  low bitwidth gradients.
\newblock \emph{arXiv preprint arXiv:1606.06160}, 2016.

\end{thebibliography}

\end{document}